\documentclass{article}
\usepackage{amsmath}
\usepackage{amssymb}
\usepackage{amsthm}
\usepackage{graphicx}
\usepackage{subfigure}
\usepackage{mathrsfs,mathtools}
\usepackage{bm}
\usepackage{tikz,pgfplots}
\usepackage{algorithm,algorithmic}

\newtheorem{theorem}{Theorem}
\newtheorem{lemma}[theorem]{Lemma}
\newtheorem{proposition}[theorem]{Proposition}
\newtheorem{corollary}[theorem]{Corollary}
\theoremstyle{definition}

\usepackage[colorlinks=true,pagebackref,linkcolor=magenta]{hyperref}
\usepackage[colon,sort&compress]{natbib}

\newcommand{\argmin}{\operatornamewithlimits{argmin}}
\renewcommand{\tilde}{\widetilde}
\renewcommand{\Re}{\mathbb{R}}
\begin{document}
\title{A consistent adjacency spectral embedding for stochastic blockmodel graphs}
\author{Daniel L. Sussman, Minh Tang, Donniell E. Fishkind, Carey E. Priebe\\
Johns Hopkins University, Applied Math and Statistics Department}

\maketitle
\begin{abstract}
We present a method to estimate block membership of nodes in a random graph generated by a stochastic blockmodel. We use an embedding procedure motivated by the random dot product graph model, a particular example of the latent position model. The embedding associates each node with a vector; these vectors are clustered via minimization of a square error criterion. We prove that this method is consistent for assigning nodes to blocks, as only a negligible number of nodes will be mis-assigned. We prove consistency of the method for directed and undirected graphs. The consistent block assignment makes possible consistent parameter estimation for a stochastic blockmodel. We extend the result in the setting where the number of blocks grows slowly with the number of nodes. Our method is also computationally feasible even for very large graphs. We compare our method to Laplacian spectral clustering through analysis of simulated data and a graph derived from Wikipedia documents.
\end{abstract}

\section{Background and Overview}
Network analysis is rapidly becoming a key tool in the analysis of modern datasets in fields ranging from neuroscience to sociology to biochemistry.  In each of these fields, there are objects, such as neurons, people, or genes, and there are relationships between objects, such as synapses, friendships, or protein interactions. The formation of these relationships can depend on attributes of the individual objects as well as higher order properties of the network as a whole. Objects with similar attributes can form communities with similar connective structure, while unique properties of individuals can fine tune the shape of these relationships. Graphs encode the relationships between objects as edges between nodes in the graph.

Clustering objects based on a graph enables identification of communities and objects of interest as well as illumination of overall network structure. 
Finding optimal clusters is difficult and will depend on the particular setting and task. Even in moderately sized graphs, the number of possible partitions of nodes is enormous, so a tractable search strategy is necessary.
Methods for finding clusters of nodes in graphs are many and varied, with origins in physics, engineering, and statistics; \cite{Fortunato2010} and \cite{Fjallstrom1998} provide comprehensive reviews of clustering techniques.
In addition to techniques motivated by heuristics based on graph structure, others have attempted to fit statistical models with inherent community structure to a graph. \citep{Nowicki2001,Handcock2007,snijders97:_estim,Airoldi2008}. 

These statistical models use random graphs to model relationships between objects; \cite{Goldenberg2009} provides a review of statistical models for networks. A graph consists of a set of nodes, representing the objects, and a set of edges, representing relationships between the objects. The edges can be either directed (ordered pairs of nodes) or undirected (unordered pairs of nodes). In our setting, the node set is fixed and the set of edges is random. 

\cite{hoff02:_laten} proposed what they call a latent space model for random graphs. Under this model each node is associated with a latent random vector. There may also be additional covariate information which we do not consider in this work.
The vectors are independent and identically distributed and the probability of an edge between two nodes depends only on their latent vectors.
Conditioned on the latent vectors, the presence of each edge is an independent Bernoulli trial.

One example of a latent space model is the random dot product graph (RDPG) model \citep{young07:_random}. Under the RDPG model, the probability an edge between two nodes is present is given by the dot product of their respective latent vectors.
For example, in a social network with edges indicating friendships, the components of the vector may be interpreted as the relative interest of the individual in various topics. The magnitude of the vector can be interpreted as how talkative the individual is, with more talkative individuals more likely to form relationships. Talkative individuals interested in the same topics are most likely to form relationships while individuals who do not share interests are unlikely to form relationships.

We present an embedding motivated by the RDPG model which uses a decomposition of a low rank approximation of the adjacency matrix. The decomposition gives an embedding of the nodes as vectors in a low dimensional space. This embedding is similar to embeddings used in spectral clustering but operates directly on the adjacency matrix rather than a Laplacian. We discuss a relationship between spectral clustering and our work in Section~\ref{sec:disc}.

Our results are for graphs generated by a stochastic blockmodel \citep{Holland1983,Wang1987}. In this model, each node is assigned to a block, and the probability of an edge between two nodes depends only on their respective block memberships; in this manner two nodes in the same block are stochastically equivalent. In the context of the latent space model, all nodes in the same block are assigned the same latent vector. An advantage of this model is the clear and simple block structure, where block membership is determined solely by the latent vector.

Given a graph generated from a stochastic blockmodel, our primary goal is to accurately assign all of the nodes to their correct blocks. 
Algorithm~\ref{alg:procedure} gives the main steps of our procedure. In summary these steps involve computing the singular value decomposition of the adjacency matrix, reducing the dimension, coordinate-scaling the singular vectors by the square root of their singular value and, finally, clustering via minimization of a square error criterion. We note that Step 4 in the procedure is a mathematically convenient stand in for what might be used in practice. Indeed, the standard $K$-means algorithm approximately minimizes the square error and we use $K$-means for evaluating the procedure empirically. 
This paper shows that the node assignments returned by Algorithm~\ref{alg:procedure} are consistent.


Consistency of node assignments means that the proportion of mis-assigned nodes goes to zero (probabilistically) as the number of nodes goes to infinity. Others have already shown similar consistency of node assignments.
\cite{snijders97:_estim} provided an algorithm to consistently assign nodes to blocks under the stochastic blockmodel for two blocks, and later \cite{Condon1999} provided a consistent method for equal sized blocks. 
\cite{Bickel2009} showed that maximizing the Newman--Girvan modularity \citep{Newman2004} or the likelihood modularity provides consistent estimation of block membership. 
\cite{Choi2010} used likelihood methods to show consistency with rapidly growing numbers of blocks. 

Maximizing modularities and likelihood methods are both computationally difficult, but provide theoretical results for rapidly growing numbers of blocks. Our method is related to that of \cite{McSherry2001}, in that we consider a low rank approximation the adjacency matrix, but their results do not provide consistency of node assignments.
\cite{rohe10:_spect_stoch_block_model} used spectral clustering to show consistent estimation of block partitions with growing number of blocks; in this paper we demonstrate that for both directed and undirected graphs, our proposed embedding allows for accurate block assignment in a stochastic blockmodel. 
These matrix decomposition methods are computationally feasible, even for graphs with a large number of nodes. 

\begin{algorithm}[t]
\begin{algorithmic}
\STATE \textbf{Input}: $\mathbf{A}\in \{0,1\}^{n\times n}$

\STATE \textbf{Parameters}: $d\in\{1,2,\dotsc,n\}$, $K\in\{2,3,\dotsc,n\}$

\STATE \textit{Step 1}: Compute the singular value decomposition, $\mathbf{A}=\tilde{\mathbf{U}}'\tilde{\bm{\Sigma}}'\tilde{\mathbf{V}}'^T$. Let $\tilde{\bm{\Sigma}}'$ have decreasing main diagonal.

\STATE \textit{Step 2}: Let $\tilde{\mathbf{U}}$ and $\tilde{\mathbf{V}}$ be the first $d$ columns of $\tilde{\mathbf{U}}'$ and $\tilde{\mathbf{V}}'$, respectively, and let $\tilde{\bm{\Sigma}}$ be the sub-matrix of $\tilde{\bm{\Sigma}}'$ given by the first $d$ rows and columns.

\STATE \textit{Step 3}: Define $\tilde{\mathbf{Z}}=[\tilde{\mathbf{U}}\tilde{\bm{\Sigma}}^{1/2}|\tilde{\mathbf{V}}\tilde{\bm{\Sigma}}^{1/2}]\in\Re^{n\times 2d}$ to be the concatenation of the coordinate-scaled singular vector matrices.

\STATE \textit{Step 4}: Let $(\hat{\bm{\psi}},\hat{\tau}) = \argmin_{\bm{\psi};\tau} \sum_{u=1}^n \|\tilde{Z}_u - \psi_{\tau(u)}\|_2^2$ give the centroids and block assignments, where $\tilde{Z}_u$ is the $u^{\text{th}}$ row of $\tilde{\mathbf{Z}}$, $\hat{\bm{\psi}}\in\Re^{K\times d}$ are the centroids and $\hat{\tau}$ is a function from $[n]$ to $[K]$.

\RETURN $\hat{\tau}$, the block assignment function, 
\end{algorithmic}
\caption{The adjacency spectral clustering procedure for directed graphs.}
\label{alg:procedure}
\end{algorithm}

The remainder of the paper is organized as follows. In Section~\ref{sec:modAlg} we formally present the stochastic blockmodel, the random dot product graph model and our adjacency spectral embedding.
In Section~\ref{sec:mainResults} we state and prove our main theorem, and in Section~\ref{sec:extension} we present some useful Corollaries. Sections~\ref{sec:modAlg}--\ref{sec:extension} focus only on directed random graphs; in Section~\ref{sec:undirected} we present model and results for undirected graphs. In Section~\ref{sec:results} we present simulations and empirical analysis to illustrate the performance of the algorithm. Finally, in section~\ref{sec:disc} we discuss further extensions to the theorem. In the appendix, we prove some key technical results to prove our main theorem.

\section{Model and Embedding} \label{sec:modAlg}
First, we adopt the following conventions. For a matrix $\mathbf{M}\in\mathbb{R}^{n\times m}$, entry $i,j$ is denoted by $\mathbf{M}_{ij}$. Row $i$ is denoted $M_i^T\in\mathbb{R}^{1\times d}$, where $M_i$ is a column vector. Column $j$ is denoted as $\mathbf{M}_{\cdot j}$ and occasionally we refer to row $i$ as $\mathbf{M}_{i\cdot}$. 

The node set is $[n]=\{1,2,\dotsc,n\}$. For directed graphs edges are ordered pairs of elements in $[n]$. 
For a random graph, the node set is fixed and the edge set is random. The edges are encoded in an adjacency matrix $\mathbf{A}\in\{0,1\}^{n\times n}$. For directed graphs, the entry $\mathbf{A}_{uv}$ is 1 or 0 according as an edge from node $u$ to node $v$ is present or absent in the graph.
We consider graphs with no loops, meaning $\mathbf{A}_{uu}=0$ for all $u\in[n]$.

\subsection{Stochastic Blockmodel} \label{sec:sbm}
Our results are for random graphs distributed according to a stochastic blockmodel \citep{Holland1983,Wang1987}, where each node is a member of exactly one block and the probability of an edge from node $u$ to node $v$ is determined by the block memberships of nodes $u$ and $v$ for all $u,v\in[n]$. The model is parametrized by $\mathbf{P}\in[0,1]^{K\times K}$, and $\rho\in(0,1)^{K}$ with $\sum_{i=1}^{K}\rho_i =1$. $K$ is the number of blocks, which are labeled $1,2,\dotsc,K$. The block memberships of all nodes are determined by the random block membership function $\tau:[n]\mapsto[K]$. 
For all nodes $u\in[n]$ and blocks $i\in[K]$,  $\tau(u)=i$ would mean node $u$ is a member of block $i$; 
node memberships are independent with $\mathbb{P}[\tau(u)=i]=\rho_i$.

The entry $\mathbf{P}_{ij}$ gives the probability of an edge from a node in block $i$ to a node in block $j$ for each $i,j\in[K]$. Conditioned on $\tau$, the entries of $\mathbf{A}$ are independent, and $\mathbf{A}_{uv}$ is a Bernoulli random variable with parameter $\mathbf{P}_{\tau(u),\tau(v)}$ for all $u\neq v\in[n]$.
This gives
\begin{equation}
\begin{split}
\mathbb{P}[\mathbf{A}|\tau] &= \prod_{u \neq v}
  \mathbb{P}[\mathbf{A}_{uv}\,|\, \tau(u), \tau(v)] \\
  &= \prod_{u \neq v} (\mathbf{P}_{\tau(u),\tau(v)})^{\mathbf{A}_{uv}}(1-\mathbf{P}_{\tau(u),\tau(v)})^{1-\mathbf{A}_{uv}},
  \end{split}
\label{eq:PAsbm}
\end{equation}
with the product over all ordered pairs of nodes.

The row $\mathbf{P}_{i\cdot}$ and column $\mathbf{P}_{\cdot i}$ determine the probabilities of the presence of edges incident to a node in block $i$. 
In order that the blocks be distinguishable, we require that different blocks have distinct probabilities so that either $\mathbf{P}_{i\cdot}\neq \mathbf{P}_{j\cdot}$ or $\mathbf{P}_{\cdot i}\neq \mathbf{P}_{\cdot j}$ for all $i\neq j\in[K]$.

Theorem~\ref{thm:main} shows that using our embedding (Section~\ref{sec:embedding})
and a mean square error clustering criterion (Section~\ref{sec:clustCrit}), we are able to accurately assign nodes to blocks, for all but a negligible number of nodes, for graphs distributed according to a stochastic blockmodel. 

\subsection{Random Dot Product Graphs} \label{sec:rdpg}
We present the random dot product graph (RDPG) model to motivate our embedding technique (Section~\ref{sec:embedding}) and provide a second parametrization for stochastic blockmodels (Section ~\ref{sec:sbmRdpg}). Let $\mathbf{X},\mathbf{Y}\in\mathbb{R}^{n\times d}$ be such that $\mathbf{X}=[X_1,X_2,\dotsc,X_n]^T$ and $\mathbf{Y}=[Y_1,Y_2,\dotsc,Y_n]^T$, where $X_u,Y_u\in\mathbb{R}^d$ for all $u\in[n]$. 
The matrices $\mathbf{X}$ and $\mathbf{Y}$ are random and satisfy  $\mathbb{P}[\langle X_u, Y_v\rangle \in [0,1]]=1$ for all $u,v\in[n]$. 
Conditioned on $\mathbf{X}$ and $\mathbf{Y}$, the entries of the adjacency matrix $\mathbf{A}$ are independent and $\mathbf{A}_{uv}$ is a Bernoulli random variable with parameter $\langle X_u, Y_v\rangle$ for all $u\neq v\in[n]$. This gives 
\begin{equation}
  \label{eq:rdpgP}
  \begin{split}
  \mathbb{P}[ \mathbf{A} \,|\, \mathbf{X},\mathbf{Y}] &= \prod_{u \neq v}
  \mathbb{P}[\mathbf{A}_{uv}\,|\, X_u, Y_v] \\
  &= \prod_{u \neq
    v} \langle X_u, Y_v \rangle^{\mathbf{A}_{uv}}(1-\langle X_u, Y_v \rangle)^{1-\mathbf{A}_{uv}},
    \end{split}
\end{equation}
where the product is over all ordered pairs of nodes.

\subsection{Embedding}\label{sec:embedding}

The RDPG model motivates the following embedding.
By an embedding of an adjacency matrix $\mathbf{A}$ we mean
\begin{equation}
(\tilde{\mathbf{X}},\tilde{\mathbf{Y}}) = \argmin_{(\mathbf{X}^\dagger,\mathbf{Y}^\dagger)\in\mathbb{R}^{n\times d}\times\mathbb{R}^{n\times d}} \| \mathbf{A}-\mathbf{X}^\dagger{\mathbf{Y}^\dagger}^T\|_F
\label{eq:argminAD}
\end{equation}
where $d$, the target dimensionality of the embedding, is fixed and known and $\|\nolinebreak\cdot\nolinebreak\|_F$ denotes the Frobenius norm. Though $\tilde{\mathbf{X}}\tilde{\mathbf{Y}}^T$ may be a poor approximation of $\mathbf{A}$, Theorems~\ref{thm:main} and \ref{thm:mainUD} show that such an embedding provides a representation of the nodes which enables clustering of the nodes provided the random graph is distributed according to a stochastic blockmodel. In fact, if a graph is distributed according to an RDPG model then a solution to Eqn.~\ref{eq:argminAD} provides an estimate of the latent vectors given by $\mathbf{X}$ and $\mathbf{Y}$. We do not explore properties of this estimate but instead focus on the stochastic blockmodel.

\cite{Eckart1936} provided the following solution to Eqn.~\ref{eq:argminAD}. Let $\mathbf{A}=\tilde{\mathbf{U}}'\tilde{\bm{\Sigma}}'\tilde{\mathbf{V}}'^T$ be the singular value decomposition of $\mathbf{A}$, where $\tilde{\mathbf{U}}',\tilde{\mathbf{V}}' \in \mathbb{R}^{n\times n}$ are orthogonal and $\tilde{\bm{\Sigma}}'\in\mathbb{R}^{n\times n}$ is diagonal, with diagonals $\sigma_1(\mathbf{A})\geq \sigma_2(\mathbf{A})\geq \dotsb \geq \sigma_n(\mathbf{A}) \geq 0$, the singular values of $\mathbf{A}$. Let $\tilde{\mathbf{U}}\in\mathbb{R}^{n\times d}$ and $\tilde{\mathbf{V}}\in\mathbb{R}^{n\times d}$  be the first $d$ columns of $\tilde{\mathbf{U}}'$ and $\tilde{\mathbf{V}}'$, respectively, and let $\tilde{\bm{\Sigma}}\in\mathbb{R}^{d\times d}$ be the diagonal matrix with diagonals $\sigma_1(\mathbf{A}),\dotsc,\sigma_d(\mathbf{A})$. Eqn.~\ref{eq:argminAD} is solved by $\tilde{\mathbf{X}}=\tilde{\mathbf{U}}\tilde{\bm{\Sigma}}^{1/2}$ and $\tilde{\mathbf{Y}}=\tilde{\mathbf{V}}\tilde{\bm{\Sigma}}^{1/2}$. 

We refer to $(\tilde{\mathbf{X}},\tilde{\mathbf{Y}})$ as the ``scaled adjacency spectral embedding'' of $\mathbf{A}$. We refer to $(\tilde{\mathbf{U}},\tilde{\mathbf{V}})$ as the ``unscaled adjacency spectral embedding'' of $\mathbf{A}$. 
The adjacency spectral embedding is similar to an embedding which is presented in \citet{priebe11:_vertex}. It is also similar to spectral clustering where the decomposition is on the normalized graph Laplacian.

Theorem~\ref{thm:main} uses a clustering of the unscaled adjacency spectral embedding of $\mathbf{A}$ while Corollary~\ref{cor:consXY} extends the result to clustering on the scaled adjacency spectral embedding. Though this embedding is proposed for embedding an adjacency matrix, we use the same procedure to embed other matrices.

\subsection{Clustering Criterion} \label{sec:clustCrit}
We prove that for a graph distributed according to the stochastic blockmodel, we can use the following clustering criterion on the adjacency spectral embedding of $\mathbf{A}$ to accurately assign nodes to blocks.
Let $\mathbf{Z}\in\mathbb{R}^{n\times m}$. We use the following mean square error criterion for  clustering the rows of $\mathbf{Z}$  into $K$ blocks,
\begin{equation}
(\hat{\bm{\psi}},\hat{\tau}) = \argmin_{\bm{\psi};\tau} \sum_{u=1}^n \|Z_u - \psi_{\tau(u)}\|_2^2,
\label{eq:minC}
\end{equation}
where $\hat{\bm{\psi}}\in\mathbb{R}^{K\times m}$,  $\hat{\psi}_i\in\mathbb{R}^{m}$ gives the centroid of block $i$ and $\hat{\tau}:[n]\mapsto[K]$ is the block assignment function. 

Again, note that other computationally less expensive criterion can also be quite effective. Indeed, in Section~\ref{sec:sim}, we achieve misclassification rates which are empirically better than our theoretical bounds using the $K$-means clustering algorithm, which only attempts to solve Eqn.~\ref{eq:minC}. Additionally, other clustering algorithms may prove useful in practice though presently we do not investigate these procedures.

\subsection{Stochastic Blockmodel as RDPG Model}\label{sec:sbmRdpg}
We present another parametrization of a stochastic blockmodel corresponding to the RDPG model. Suppose we have a stochastic blockmodel with $\mathrm{rank}(\mathbf{P})=d$. Then there exist $\bm{\nu},\bm{\mu}\in\mathbb{R}^{K\times d}$ such that $\mathbf{P}=\bm{\nu}\bm{\mu}^T$ and by definition $\mathbf{P}_{ij}=\langle \nu_i, \mu_j \rangle$. Let $\tau:[n]\mapsto[K]$ be the random block membership function. 
%

Let $\mathbf{X}\in\mathbb{R}^{n\times d}$ and $\mathbf{Y}\in\mathbb{R}^{n\times d}$ have row $u$ given by $X_u^T = \nu_{\tau(u)}^T$ and $Y_u^T=\mu_{\tau(u)}^T$, respectively, for all $u$. Then we have
\begin{equation}
\mathbb{P}[\mathbf{A}_{uv}=1]=\mathbf{P}_{\tau(u),\tau(v)}=\langle\nu_{\tau(u)},\mu_{\tau(v)}\rangle=\langle X_u, Y_v\rangle.
\label{eq:PsbmRdpg}
\end{equation} 
In this way, the stochastic blockmodel can be parametrized by $\bm{\nu},\bm{\mu}\in\mathbb{R}^{K\times d}$ and $\rho$ provided that $(\bm{\nu\mu}^T)_{ij}\in[0,1]$ for all $i,j\in[K]$. This viewpoint proves valuable in the analysis and clustering of the adjacency spectral embedding.

Importantly, the distinctness of rows or columns in $\mathbf{P}$ is equivalent to the distinctness of the rows of $\bm{\nu}$ or $\bm{\mu}$. (Indeed note, that for $i\neq j$, $\mathbf{P}_{i\cdot}-\mathbf{P}_{j\cdot}=0$ if and only if $(\nu_i^T-\nu_j^T)\bm{\mu}=0$, but $\mathrm{rank}(\bm{\mu})=d$ so $\nu_i^T=\nu_j^T$. Similarly, $\mathbf{P}_{\cdot i}=\mathbf{P}_{\cdot j}$ if and only if $\mu_i=\mu_j$.) Also note, we can take $(\bm{\nu},\bm{\mu})$ as the adjacency spectral embedding of $\mathbf{P}$ with target dimensionality $\mathrm{rank}(\mathbf{P})$ to get such a representation from any given $\mathbf{P}$.

\section{Main Results} \label{sec:mainResults}

\subsection{Notation} \label{sec:notation}

We use the following notation for the remainder of this paper. Let $\mathbf{P}\in [0,1]^{K\times K}$ and $\rho\in(0,1)^K$ be a vector with positive entries summing to unity. Suppose $\mathrm{rank}(\mathbf{P})=d$. 
Let $\bm{\nu}\bm{\mu}^T=\mathbf{P}$ with $\bm{\nu},\bm{\mu}\in\mathbb{R}^{K\times d}$. 
We now define the following constants not depending on $n$:
\begin{itemize}
        \item $\alpha>0$ such that all eigenvalues of $\bm{\nu}^T\bm{\nu}$ and $\bm{\mu}^T\bm{\mu}$ are greater than $\alpha$;
        \item $\beta>0$ such that $\beta<\|\nu_i-\nu_j\|$ or $\beta<\|\mu_i-\mu_j\|$ for all $i\neq j$;
        \item $\gamma>0$ such that $\gamma<\rho_i$ for all $i\in[K]$.
\end{itemize}

We consider a sequence of 
random adjacency matrices $\mathbf{A}^{(n)}$ with node set $[n]$ for $n\in\{1,2,\dotsc\}$. The edges are distributed according to a stochastic blockmodel with parameters $\mathbf{P}$ and $\rho$.
Let $\tau^{(n)}:[n]\mapsto[K]$ be the random block membership function, which induces the matrices $\mathbf{X}^{(n)},\mathbf{Y}^{(n)}\in\mathbb{R}^{n\times d}$ as in Section~\ref{sec:sbmRdpg}. Let $n_i = |\{u:\tau(u)=i\}|$ be the size of block $i$. 

Let $\mathbf{XY}^T=\mathbf{U}\bm{\Sigma}\mathbf{V}$ be the singular value of decomposition, with $\mathbf{U},\mathbf{V}\in\Re^{n\times d}$ and $\bm{\Sigma}\in\Re^{d\times d}$, so that  $(\mathbf{U},\mathbf{V})$ is the unscaled spectral embedding of the $\mathbf{XY}^T$.
 Let $(\tilde{\mathbf{X}},\tilde{\mathbf{Y}})$ be the adjacency spectral embedding of $\mathbf{A}$ and let $(\tilde{\mathbf{U}},\tilde{\mathbf{V}})$ be the unscaled adjacency spectral embedding of $\mathbf{A}$. Finally, let $\mathbf{W}\in\mathbb{R}^{n\times 2d}$ be the concatenation  $[\mathbf{U}|\mathbf{V}]$ and similarly $\tilde{\mathbf{W}} = [\tilde{\mathbf{U}}|\tilde{\mathbf{V}}]$. 

\subsection{Main Theorem}\label{sec:mainThm}
The main contribution of this paper is the following consistency
result in terms of the estimation of the block memberships for each node based on the block assignment function $\hat{\tau}$ which assigns blocks based on $\tilde{\mathbf{W}}$. In the following,
an event occurs ``almost always'' if with probability 1 the event occurs for all but finitely many $n\in\{1,2,\dotsc\}$. 
\begin{theorem}
  \label{thm:main}
Under the conditions of Section~\ref{sec:notation}, suppose that the number of blocks $K$ and the latent vector dimension $d$ are known.  Let
  $\hat{\tau}^{(n)} \colon V \mapsto
  [K]$ be the block assignment function according to a  clustering of the rows of
  $\tilde{\mathbf{W}}^{(n)}$ satisfying Eqn.~\ref{eq:minC}. Let $\mathcal{S}_{K}$ be the set of permutations
  on $[K]$. It almost always holds that
  \begin{equation}
    \label{eq:consistency1}
     \min_{\pi \in
       \mathcal{S}_K} | \{ u\in V \colon \tau(u) \not =
     \pi(\hat{\tau}(u)) \}|\leq \frac{2^3 3^2 6}{\alpha^5\beta^2\gamma^5}\log n.
   \end{equation}
\end{theorem}

To prove this theorem, we first provide a bound on the Frobenius norm of 
${\mathbf{A}}{\mathbf{A}}^T-(\mathbf{X}\mathbf{Y}^{T})(\mathbf{X}\mathbf{Y}^{T})^T$, following \citet{rohe10:_spect_stoch_block_model}. Using this results and properties of the stochastic blockmodel, we then find a lower bound for the smallest non-zero singular value of $\mathbf{XY}^T$ and the corresponding singular value of $\mathbf{A}$. This enables us to apply the Davis-Kahan Theorem \citep{davis70} to show that the unscaled adjacency spectral embedding of $\mathbf{A}$ is approximately a rotation of the unscaled adjacency spectral embedding of $\mathbf{XY}^T$. 

Finally, we lower bound the distances between the at most $K$ distinct rows of $\mathbf{U}$ and $\mathbf{V}$. These gaps, together with the good approximation by the embedding of $A$ is sufficient to prove consistency of the mean square error clustering of the embedded vectors. Most results, except the important Proposition~\ref{prop:froConv} and the main theorem, are proved in the Appendix.

\begin{proposition}
        \label{prop:froConv}
Let $\mathbf{Q}^{(n)}\in [0,1]^{n\times n}$ be a sequence of random matrices and let $\mathbf{A}^{(n)}\in\{0,1\}^{n\times n}$ be a sequence of random adjacency matrices corresponding to a sequence of random graphs on $n$ nodes for $n\in\{1,2,\dotsc\}$. Suppose the probability of an edge from node $u$ to node $v$ is given by $\mathbf{Q}^{(n)}_{uv}$ and that the presence of edges are conditionally independent given $\mathbf{Q}^{(n)}$. Then the following holds almost always:
\begin{equation}
\|\mathbf{A}^{(n)}{\mathbf{A}^{(n)}}^T - \mathbf{Q}^{(n)}{\mathbf{Q}^{(n)}}^T \|_F \leq \sqrt{3}n^{3/2}\sqrt{\log n}.
\label{eq:froBound}
\end{equation}
\end{proposition}
\begin{proof}
For ease of exposition, we dropped the index $n$ from $\mathbf{Q}^{(n)}$. 
Note that, conditioned on $\mathbf{Q}$, $\mathbf{A}_{uw}$ and $\mathbf{A}_{vw}$ are independent Bernoulli random variables for all $w\in[n]$ provided $u\neq v$. For each $w\notin\{u,v\}$, $\mathbf{A}_{uw}\mathbf{A}_{vw}$ is a conditionally independent Bernoulli with parameter $\mathbf{Q}_{uw}\mathbf{Q}_{vw}$. For $u\neq v$, we have
\begin{equation}
\begin{split}
\mathbf{A}\mathbf{A}^T_{uv}-\mathbf{Q}\mathbf{Q}^T_{uv} =& \sum_{w\notin \{u,v\}} (\mathbf{A}_{uw}\mathbf{A}_{vw}- \mathbf{Q}_{uw} \mathbf{Q}_{vw}) \\
&- \mathbf{Q}_{uu}\mathbf{Q}_{vu}-\mathbf{Q}_{uv}\mathbf{Q}_{vv}.
\end{split}
\label{eq:A-Pij}
\end{equation}
Thus, by Hoeffding's inequality,
\begin{equation}
\begin{split}
\mathbb{P}[(\mathbf{A}\mathbf{A}^T_{uv}-\mathbf{Q}\mathbf{Q}^T_{uv})^2\geq 2(n-2)\log n+2n+4\ |\ \mathbf{Q}]\leq 2 n^{-4}.
\end{split}
\label{eq:Hoefding}
\end{equation}
We can integrate over all choices of $\mathbf{Q}$ so that Eqn.~\ref{eq:Hoefding} holds unconditionally.

For the diagonal entries,  $(\mathbf{A}\mathbf{A}^T_{uu}-\mathbf{Q}\mathbf{Q}^T_{uu})^2 \leq n^2$ always.
The diagonal terms and the $2n+4$ terms from equation~\ref{eq:Hoefding} all sum to at most $3n^3+4n^2\leq n^3\log n$ for $n$ large enough. 
Combining these inequalities we get the inequality
\begin{equation}
\mathbb{P}[\|\mathbf{A}\mathbf{A}^T - \mathbf{Q}\mathbf{Q}^T\|_F^2 \geq 3n^3 \log n] \leq 2 n^{-2}.
\label{eq:froDiff1}
\end{equation}
Applying the Borel-Cantelli Lemma gives the result.
\end{proof}

Taking $\mathbf{Q}=\mathbf{XY}^T$ gives the following immediate corollary.
\begin{corollary} \label{cor:froBound}
It almost always holds that 
\begin{equation}
\|\mathbf{AA}^T- \mathbf{XY}^T(\mathbf{XY}^T)^T\|_F \leq \sqrt{3}n^{3/2}\sqrt{\log n}
\label{eq:froBound2}
\end{equation}
 and 
 \begin{equation}
\|\mathbf{A}^T\mathbf{A}- (\mathbf{XY}^T)^T\mathbf{XY}^T\|_F\leq \sqrt{3}n^{3/2}\sqrt{\log n}.
\label{eq:froBound3}
\end{equation} 
\end{corollary}

The next two results provide bounds on the singular values of $\mathbf{XY}^T$ and $\mathbf{A}$ based on lower bounds for the eigenvalues of $\mathbf{P}$ and the block membership probabilities.
\begin{lemma}
\label{lem:singBoundXY}
It almost always holds that $\alpha\gamma n\leq \sigma_d(\mathbf{XY}^T)$ and it always holds that $\sigma_{d+1}(\mathbf{XY}^T)=0$ and $\sigma_1(\mathbf{XY^T})\leq n$.
\end{lemma}

\begin{corollary} \label{cor:singBoundA}
It almost always holds that 
\begin{equation}
\alpha\gamma n\leq \sigma_d(\mathbf{A}) \text{ and } \sigma_{d+1}(\mathbf{A})\leq 3^{1/4}n^{3/4}\log^{1/4}n
\label{eq:singBoundA}
\end{equation}
and it always holds that $\sigma_1(\mathbf{A})\leq n$.
\end{corollary}
We note that Corollary~\ref{cor:singBoundA} immediately suggests a consistent estimator of the rank of $\mathbf{XY}^T$ given by $\hat{d}=\max\{d':\sigma_{d'}(\mathbf{A})>3^{1/4}n^{3/4}\log^{1/4}n\}$. Presently we do not investigate the use of this estimator and assume that the $d=\mathrm{rank}(\mathbf{P})$ is known.

The  following is the version of the Davis-Kahan Theorem \citep{davis70} as stated in \cite{rohe10:_spect_stoch_block_model}.

\begin{theorem}[Davis and Kahan]\label{thm:DK}
Let $\mathbf{H},\mathbf{H}' \in \Re^{n \times n}$ be symmetric,
suppose ${\mathcal S} \subset \Re$ is an interval, and suppose for
some positive integer $d$ that $\mathbf{W},\mathbf{W}'\in \Re^{n\times d}$ are such that the columns of  $\mathbf{W}$ form an orthonormal basis  for the sum of the eigenspaces of $\mathbf{H}$ associated with the eigenvalues of $\mathbf{H}$ in $\mathcal{S}$ and that the columns of  $\mathbf{W}'$ form an orthonormal basis  for the sum of the eigenspaces of $\mathbf{H}'$ associated with the eigenvalues of $\mathbf{H}'$ in $\mathcal{S}$. Let $\delta$ be the minimum distance between
any eigenvalue of $\mathbf{H}$ in ${\mathcal S}$ and any
eigenvalue of $\mathbf{H}$ not in ${\mathcal S}$.
Then there exists an orthogonal matrix $\mathbf{R} \in \Re^{d \times d}$ such that
$\| \mathbf{W} \mathbf{R}- \mathbf{W}' \|_F \leq \frac{\sqrt{2}}{\delta}\|\mathbf{H}-\mathbf{H}'\|_F$.
\end{theorem}
For completeness, we provide a brief discussion of this important result in Appendix~\ref{sec:DK}.
Applying Theorem~\ref{thm:DK} and Lemma~\ref{lem:singBoundXY} to $\mathbf{AA}^T$ and
$\mathbf{XY}^T(\mathbf{XY}^T)^T$, we have the following result. 
\begin{lemma}
  \label{lem:1}
  It almost always holds that there exists an orthogonal matrix $\mathbf{R}\in\mathbb{R}^{2d\times2d}$ such that $\|\mathbf{WR}-\tilde{\mathbf{W}}\|\leq \sqrt{2}\frac{\sqrt{6}}{\alpha^2\gamma^2}\sqrt{\frac{\log n}{n}}$.
\end{lemma}

Recall that $\mathbf{XY}^T=\mathbf{U}\bm{\Sigma}\mathbf{V}^T$. We now provide bounds for the gaps between the at most $K$ distinct rows of $\mathbf{U}$ and $\mathbf{V}$.

\begin{lemma} It almost always holds that, for all $u,v$ such that $X_u\neq X_v$, $\|U_u-U_v\| \geq
\beta \sqrt{\alpha \gamma}n^{-1/2}$. Similarly, for all $Y_u\neq Y_v$, $\|V_u-V_v\|\geq \beta\sqrt{\alpha \gamma}n^{-1/2}$. As a result, $\|W_u-W_v\|\geq \beta\sqrt{\alpha\gamma}n^{-1/2}$ for all $u,v$ such that $\tau(u)\neq \tau(v)$.
 \label{lem:eigVecSep}
\end{lemma}

We now have the necessary ingredients to show our main result. 

\begin{proof}[Proof of Theorem~\ref{thm:main}]
Let $\hat{\bm{\psi}}$ and $\hat{\tau}$ satisfy the clustering criterion for $\tilde{\mathbf{W}}$ (where $\tilde{\mathbf{W}}=[\tilde{\mathbf{U}}|\tilde{\mathbf{V}}]$ takes the role of $\mathbf{Z}$ in Section~\ref{sec:clustCrit}). Let $\mathbf{C}\in\mathbb{R}^{n\times 2d}$ have row $u$ given by $C_u=\hat{\psi}_{\tau(u)}$. Then Equation~\ref{eq:minC} gives that $\|\mathbf{C}-\tilde{\mathbf{W}}\|_F \leq \|\mathbf{WR}-\tilde{\mathbf{W}}\|_F$ as $\mathbf{W}$ has at most $K$ distinct rows. Thus, Lemma~\ref{lem:1} gives that 
\begin{equation}
\begin{split}
\|\mathbf{C}-\mathbf{WR}\|_F& \leq \|\mathbf{C}-\tilde{\mathbf{W}}\|_F + \|\tilde{\mathbf{W}}-\mathbf{WR}\|_F\\
&\leq 2^{3/2}\frac{\sqrt{6}}{\alpha^2\gamma^2}\sqrt{\frac{\log n}{n}}.
\end{split}
\label{eq:boundC-WR}
\end{equation}
Let $\mathcal{B}_1,\mathcal{B}_2,\dotsc,\mathcal{B}_K$ be balls of radius $r=\frac{\beta}{3}\sqrt{\alpha\gamma} n^{-1/2}$ each centered around the $K$ distinct rows of $\mathbf{W}$. By Lemma~\ref{lem:eigVecSep}, these balls are almost always disjoint.

Now note that almost always the number of rows $u$ such that $\|C_u-W_u \mathbf{R}\|> r$ is at most  $\frac{2^3 3^2 6}{\alpha^5\beta^2\gamma^5}\log n$. If this were not so then infinitely often we would have 
\begin{equation}
\begin{split}
\|\mathbf{C}-\mathbf{WR}\|_F &> \frac{2^3 3^2 6}{\alpha^5\beta^2\gamma^5}\log n \frac{\beta}{3}\sqrt{\alpha\gamma} n^{-1/2}\\
&=2^{3/2}\frac{\sqrt{6}}{\alpha^2\gamma^2}\sqrt{\frac{\log n}{n}},
\end{split}
\label{eq:badBoys1}
\end{equation}
in contradiction to Eqn.~\ref{eq:boundC-WR}. Since  $n_i>\gamma n>\frac{2^3 3^2 6}{\alpha^5\beta^2\gamma^5}\log n$ almost always, each ball $\mathcal{B}_i$ can contain exactly one of the $K$ distinct rows of $\mathbf{C}$. This gives the number of misclassifications as $\frac{2^3 3^2 6}{\alpha^5\beta^2\gamma^5}\log n$ as desired. 
\end{proof}

This gives that a clustering of the concatenation of the matrices $\tilde{\mathbf{U}}$ and $\tilde{\mathbf{V}}$ from the singular value decomposition gives an accurate block assignment. One may also cluster the scaled singular vectors given by $\tilde{\mathbf{X}}$ and $\tilde{\mathbf{Y}}$ without a change in the order of the number of misclassifications.

\section{Extensions} \label{sec:extension}
\begin{corollary} \label{cor:consXY}
Under the conditions of Theorem~\ref{thm:main}, let $\hat{\tau}:V\to[K]$ be a clustering of $\tilde{\mathbf{Z}}=[\tilde{\mathbf{X}}|\tilde{\mathbf{Y}}]$. Then it almost always holds that
\begin{equation}
\min_{\pi\in \mathcal{S}_K} |\{u\in V: \pi(\hat{\tau}(u))\neq \tau(u)\}| \leq \frac{2^3 3^2 6}{\alpha^6\beta^2\gamma^6}\log n.
\label{eq:consistency2}
\end{equation}
\end{corollary}
The proof relies on the fact that the square root of the singular values are all of the same order and differ by a multiplicative factor of at most $\sqrt{\alpha\gamma}$.

We now present consistent estimators of the parameters $\mathbf{P}$ and $\rho$ for the stochastic blockmodel. Consider the following estimates 
\begin{equation}
\hat{n}_k = |\{u:\hat{\tau}(u)=k\}|,\quad \hat{\rho}_k = \dfrac{\hat{n}_k}{n} 
\label{eq:nHatRhoHat}
\end{equation}
and 
\begin{equation}
\hat{\mathbf{P}}_{ij} = \begin{dcases}
\frac{1}{\hat{n}_i\hat{n}_j} \sum_{(u,v)\in\hat{\tau}^{-1}(i)\times\hat{\tau}^{-1}(j)} \mathbf{A}_{uv}, &\text{if } i\neq j \text{ or,}\\
\frac{1}{\hat{n}_i^2-\hat{n}_i} \sum_{(u,v)\in\hat{\tau}^{-1}(i)\times\hat{\tau}^{-1}(j)} \mathbf{A}_{uv}, &\text{if } i=j.
\end{dcases}
\label{eq:pHat}
\end{equation}

This gives the following corollary.
\begin{corollary} 
Under the conditions of Theorem~\ref{thm:main},
\begin{align} &\min_{\pi\in\mathcal{S}_K}|\rho_{i}-\hat{\rho}_{\pi(i)}|\overset{a.s.}{\longrightarrow} 0 \label{eq:rhoCons}\\
\text{and } &\min_{\pi\in\mathcal{S}_K} |\hat{\mathbf{P}}_{\pi(i)\pi(j)}-\mathbf{P}_{ij}|\overset{a.s.}{\longrightarrow} 0 \label{eq:pCons}
\end{align}
for all $i,j\in[K]$ as $n\to\infty$.
\label{cor:PhatCons}
\end{corollary}
The proof is immediate from Theorem~\ref{thm:main} and the law of large numbers.

If we take $(\hat{\bm{\nu}},\hat{\bm{\mu}})$ to be the adjacency spectral embedding of $\hat{\mathbf{P}}$ then we also have that $\hat{\bm{\nu}}$ and $\hat{\bm{\mu}}$ provide consistent estimates for $(\bm{\nu},\bm{\mu})$, the adjacency spectral embedding of $\mathbf{P}$, in the following sense. 
\begin{corollary} \label{cor:numuHatCons}
Under the conditions of Theorem~\ref{thm:main}, with probability 1 there exists a sequence of orthogonal matrices $\mathbf{R}_1^{(n)},\mathbf{R}_2^{(n)}\in \mathbb{R}^{d\times d}$ such that
\begin{equation}
\|\hat{\bm{\nu}}-\bm{\nu}\mathbf{R}_1^{(n)}\|_F \to 0\text{ and } \|\hat{\bm{\mu}}-\bm{\mu}\mathbf{R}_2^{(n)}\|_F \to 0.
\label{eq:nuHatCons}
\end{equation}
\end{corollary}
The proof relies on applications of the Davis-Kahan Theorem in a similar way to Lemma~\ref{lem:1}. 

\section{Undirected Version}\label{sec:undirected}
We now present the undirected version of the stochastic blockmodel and state the main result. The setting and notation are from Section~\ref{sec:notation}.

For the undirected version of the stochastic blockmodel, the matrix $\mathbf{P}$ is symmetric and $\mathbf{P}_{ij}=\mathbf{P}_{ji}$ gives the probability of an edge between a node in block $i$ and a node in block $j$ for each $i,j\in[K]$. Conditioned on $\tau$, $\mathbf{A}_{uv}$ is a Bernoulli random variable with parameter $\mathbf{P}_{\tau(u),\tau(v)}$ for all $u\neq v\in[n]$. As $\mathbf{A}$ is symmetric, all entries of $\mathbf{A}$ are not independent, but the entries are independent provided two entries do not correspond to the same undirected edge.

For the undirected version a re-parametrization of the stochastic block model as a RDPG model as in Section~\ref{sec:sbmRdpg} is not always possible.
However, we can find $\bm{\nu},\bm{\mu}\in\mathbb{R}^{K\times d}$ such that $\bm{\nu\mu}^T=\mathbf{P}$ 
and $\bm{\nu}$ and $\bm{\mu}$ have equal columns up to a possible change in sign in each column. This means the rows of $\bm{\nu}$ and $\bm{\mu}$ are distinct so it is not necessary to cluster on the concatenated embeddings. 
Instead, we consider clustering the rows of $\tilde{\mathbf{U}}$ or $\tilde{\mathbf{X}}$, which gives a factor of two improvement in misclassification rate. 
\begin{theorem} \label{thm:mainUD}
Under the undirected version of the stochastic blockmodel, suppose that the number of blocks $K$ and the latent feature dimension $d$ are known.  Let $\hat{\tau}:V\mapsto [K]$ be a block assignment function according to a clustering  of the rows of $\tilde{\mathbf{U}}$ satisfying the criterion in Eqn.~\ref{eq:minC}. It almost always holds that
\begin{equation}
\min_{\pi\in\mathcal{S}_K} |\{ u\in V:\tau(u)\neq \pi(\hat{\tau}(u)) \}|\leq \frac{2^2 3^2 6}{\alpha^5\beta^2\gamma^5}\log n.
\label{eq:consUD}
\end{equation}
\end{theorem}

Corollary~\ref{cor:consXY} holds when clustering on $\tilde{\mathbf{X}}$, with the same factor of 2 improvement in  misclassification rate. Corollaries~\ref{cor:PhatCons} and  \ref{cor:numuHatCons} also hold without change. 

\section{Empirical Results} \label{sec:results}
We evaluated this procedure and compared it to the spectral clustering procedure of \cite{rohe10:_spect_stoch_block_model} for both simulated data (\S~\ref{sec:sim}) and using a Wikipedia hyperlink graph (\S~\ref{sec:realData}).

\subsection{Simulated Data}\label{sec:sim}

To illustrate the effectiveness of the adjacency spectral embedding, we simulate random undirected graphs generated from the following stochastic blockmodel:
\begin{equation}
 \mathbf{P}=\begin{pmatrix} 0.42 & 0.42 \\ 0.42 & 0.5 \end{pmatrix}\text{ and } \rho=(.6, .4)^T
\end{equation}
For each $n\in\{500,600,\dotsc,2000\}$, we simulated 100 monte carlo replicates from this model conditioned on the fact that $|\{u\in[n]: \tau(u)=i\}|=\rho_i n$ for each $i\in\{1,2\}$. In this model we assume that $d=2$ and $K=2$ are known. 

We evaluated four different embedding procedures and for each embedding we used $K$-means clustering, which attempts to iteratively find the solution to Eqn.~\ref{eq:minC}, to generate the node assignment function $\hat{\tau}$. The four embedding procedure are the scaled and unscaled adjacency spectral embedding as well as the scaled and unscaled Laplacian spectral embedding. The Laplacian spectral embedding uses the same spectral decomposition but works with the normalized Laplacian (as defined in \cite{rohe10:_spect_stoch_block_model}) rather then the adjacency matrix. The normalized Laplacian is given by  $\mathbf{L}=\mathbf{D}^{-1/2}\mathbf{A}\mathbf{D}^{-1/2}$ where $\mathbf{D}\in\mathbb{R}^{n\times n}$ is diagonal with $\mathbf{D}_{vv}=\mathrm{deg}(v)$, the degree of node $v$.

\begin{figure*}%
\begin{center}
\includegraphics[width=.9\textwidth]{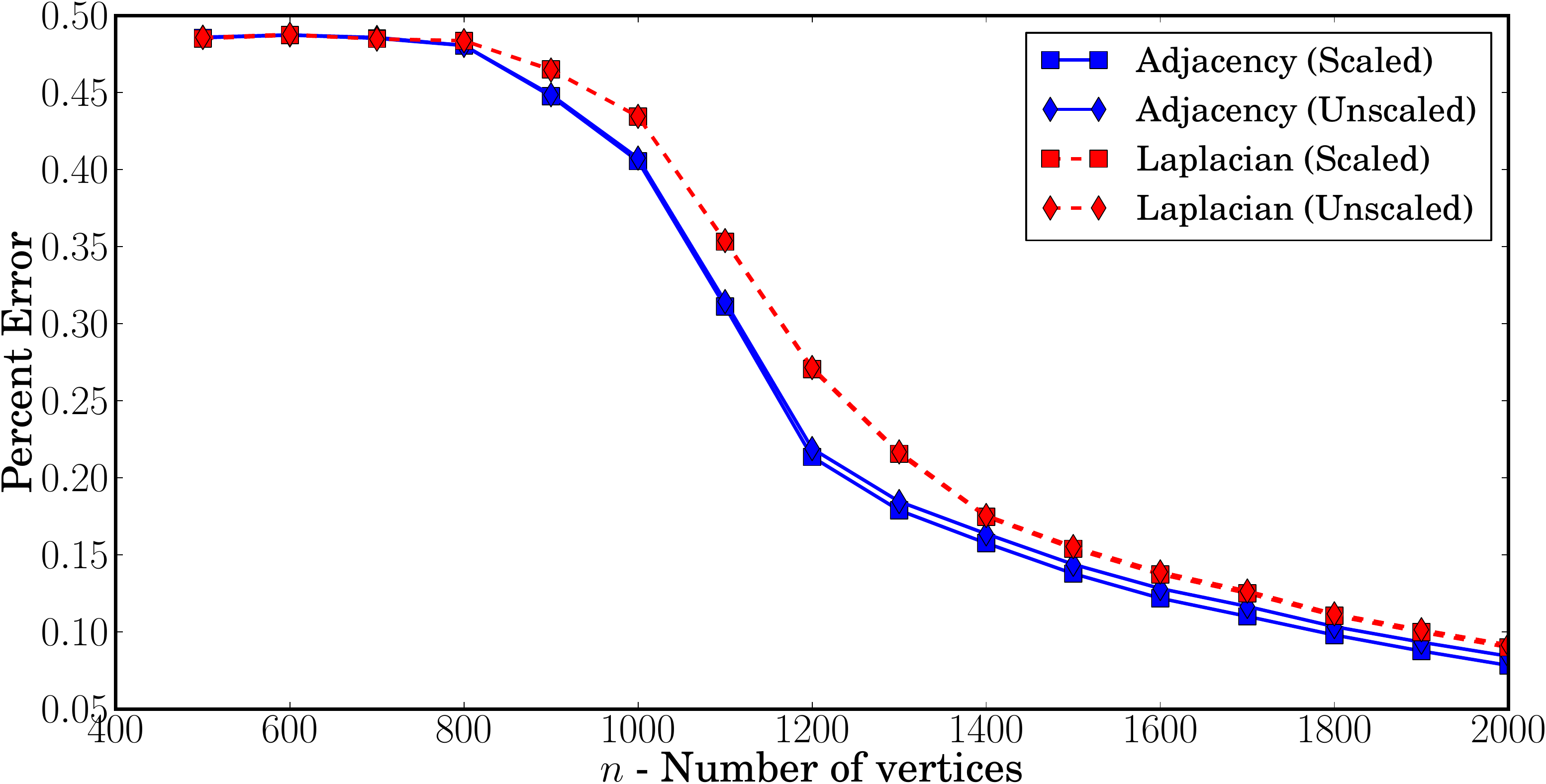}
\end{center}
\caption{Mean error for 100 monte carlo replicates using $K$-means on four different embedding procedures.}%
\label{fig:simError}%
\end{figure*}

We evaluated the performance of the node assignments by computing the percentage of mis-assigned nodes, $\min_{\pi\in S_2}|\{u\in[n]:\tau(u)\neq \pi(\hat{\tau}(u))\}|/n$, as in Eqn.~\ref{eq:consistency1}. Figure~\ref{fig:simError} demonstrates that performance of $K$-means on all four embeddings improves with increasing number of nodes. It also demonstrates (via a paired Wilcoxon test) that for these model parameters the adjacency embedding is superior to the Laplacian embeddings for large $n$. In fact, for $n\geq 1400$ we observed that for each simulated graph the scaled adjacency embedding always performed better than both Laplacian embeddings. We note that these model parameters were specifically constructed to demonstrate a case where the adjacency embedding is superior to the Laplacian embedding.

Figure~\ref{fig:simScatter} shows an example of the scaled adjacency (left) and scaled Laplacian (right) spectral embeddings. The graph has 2000 nodes and the points are colored according to their block membership. The dashed line shows the discriminant boundary given by the $K$-means algorithm with $K=2$.

\begin{figure*}
\begin{center}
\includegraphics[width=.9\textwidth]{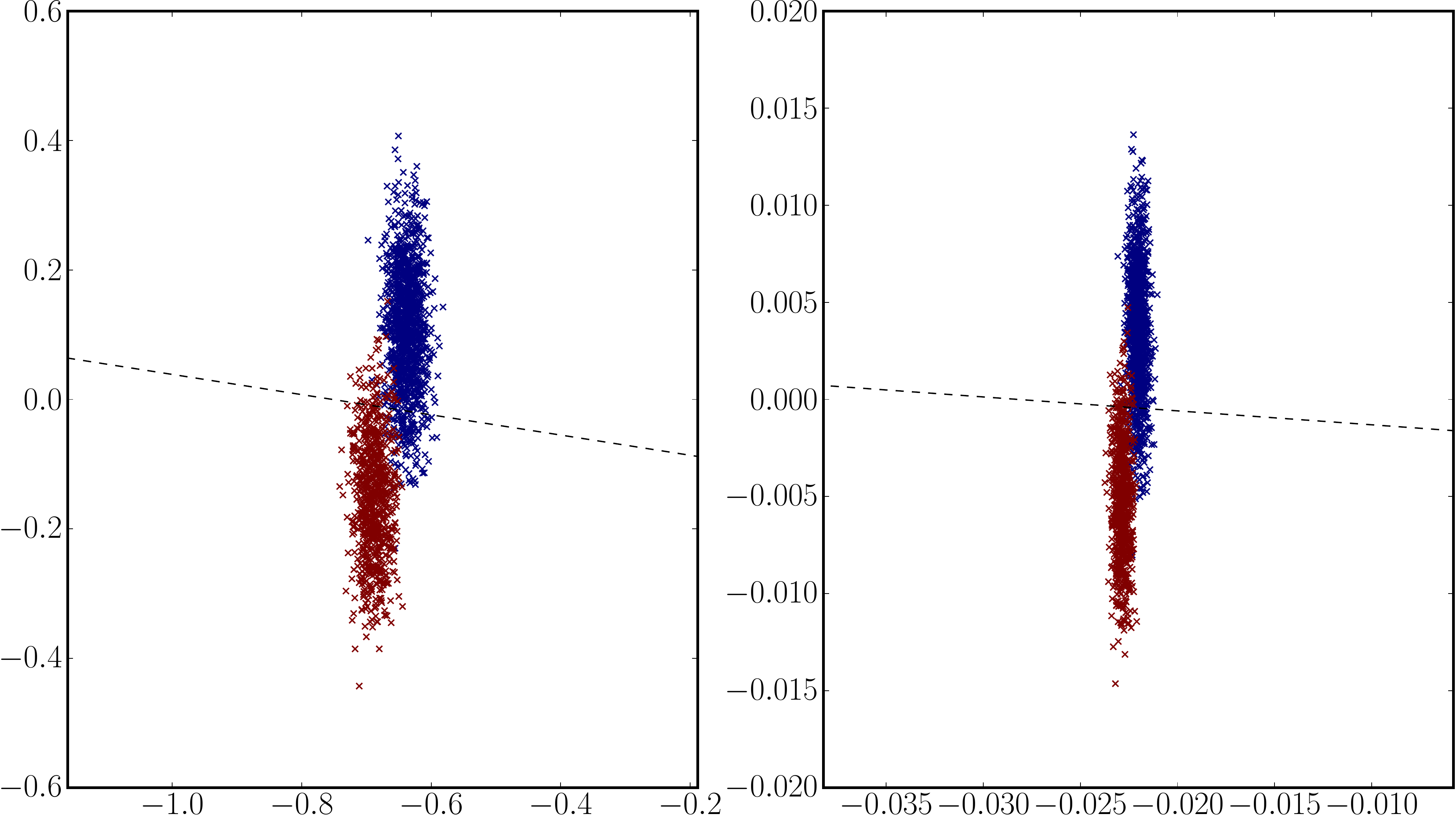}
\end{center}
\caption{Scatter plots of the scaled adjacency (left) and Laplacian (right) embeddings of a 2000 node graph. }%
\label{fig:simScatter}%
\end{figure*}

\subsection{Wikipedia Graph}\label{sec:realData}
For this data, each node in the graph corresponds to a Wikipedia page and the edges correspond to the presence of a hyperlink between two pages (in either direction). We consider this as an undirected graph. Every article within two hyperlinks of the article ``Algebraic Geometry'' was included as a node in the graph. This resulted in $n=1382$ nodes.  Additionally, each document, and hence each node, was manually labeled as one of the following: Category, Person, Location, Date and Math.

To illustrate the utility of this algorithm we embedded this graph using the scaled adjacency and Laplacian procedures. Figure~\ref{fig:wikiScatter} shows the two embeddings for $d=2$. The points are colored according to their manually assigned labels. First we note that on the whole the two embeddings look moderately different. In fact, for the adjacency embedding one can see that the orange points are well separated from the remaining data. On the other hand, with the Laplacian embedding we can see that the red points are somewhat separated from the remaining data. The dashed lines show the result boundary as determined by $K$-means with $K=2$.

To evaluate the performance we considered the 5 different tasks of identifying one block and grouping the remaining blocks together. For each of the 5 blocks, we compared each of the one-vs-all block labels to the estimated labels from $K$-means, with $K=2$, on the two embeddings. Table~\ref{tab:wg} shows the number of incorrectly assigned nodes, as in Eqn.~\ref{eq:consistency1}, as well as the adjusted Rand index \citep{Hubert1985}. The adjusted Rand index (ARI) has the property that the optimal value is 1 and a value of zero indicates the expected value if the labels were assigned randomly. 

\begin{table}[b]\small
 \begin{tabular}{|r|cc|cc|cc|cc|cc|}\hline
   & \multicolumn{2}{c|}{Category (119)}& \multicolumn{2}{c|}{Person (372)}& \multicolumn{2}{c|}{Location (270)}& \multicolumn{2}{c|}{Date (191)}& \multicolumn{2}{c|}{Math (430)} \\ 
  & Error & ARI & Error & ARI & Error & ARI & Error & ARI & Error & ARI \\\hline
A &  242  & -0.08 & 495 & -0.07 & 341 & 0.01 & \textbf{130} & \textbf{0.47} & 543 & 0.06 \\\hline
L & 299   & -0.02 & 495 & -0.02 & 476 & -0.1 & 401 & -0.10 & \textbf{350} & \textbf{0.19} \\\hline
 \end{tabular}
\normalsize
\caption{One versus all comparison of each block against the estimated $K$-means block assignments with $K=2$.}
\label{tab:wg}
\end{table}
We can see from this table that $K$-means on the adjacency embedding identifies the separation of the Date block from the other four while on the Laplacian embedding $K$-means identifies the separation of the Math block from the other four. This indicates that for this data set (and indeed more generally) the choice of embedding procedure will depend greatly on the desired exploitation task.

We note that for both embeddings, the clusters generated using $K$-means, with $K=5$, poorly reflect the manually assigned block memberships. We have not investigated beyond the illustrative 2-dimensional embeddings. 

\begin{figure*}
\begin{center}
\includegraphics[width=\textwidth]{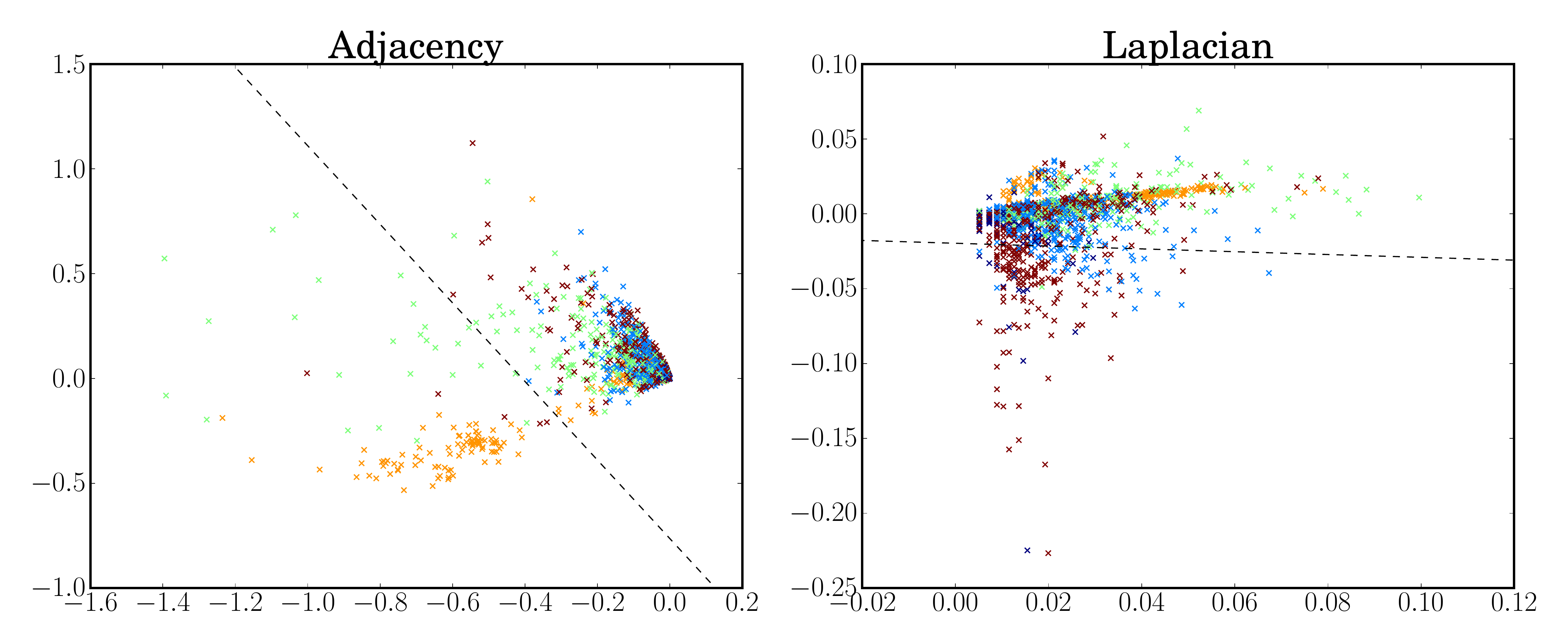}
\end{center}
\caption{Scatter plots for the Wikipedia graph. The left pane show the scaled adjacency embedding and the right pane show the scaled Laplacian embedding. Each point is colored according to the manually assigned labels. The dashed line represents the discriminant boundary determined by $K$-means with $K=2$.}%
\label{fig:wikiScatter}%
\end{figure*}

\section{Discussion} \label{sec:disc} 
Our simulations demonstrate that for a particular example of the stochastic blockmodel, the proportion of mis-assigned nodes will rapidly become small. Though our bound shows that the number of mis-assigned nodes will not grow faster than $O(\log n)$, in some instances this bound may be very loose. 
We also demonstrate that using the adjacency embedding over the Laplacian embedding can provide performance improvements in some settings. 
It is also clear from Figure~\ref{fig:simScatter} that the use of other unsupervised clustering techniques, such as Gaussian mixture modeling, will likely lead to further performance improvements.

On the Wikipedia graph, the two-dimensional embedding demonstrates that the adjacency embedding procedure provides an alternative to the Laplacian embedding and the two may have fundamentally different properties. Both the Date block and the Math block have some differentiating structure in the graph but these structures are illuminated more in one embedding then the other. This analysis suggests that further investigations into comparisons between the adjacency embedding and the Laplacian embeddings will be fruitful.

Our empirical analysis indicates that the adjacency spectral embeddings and the Laplacian spectral embeddings are strongly related while the two embeddings may emphasize different aspects of the particular graph. \cite{rohe10:_spect_stoch_block_model} used similar techniques to show consistency of block assignment on the Laplacian embedding and achieved the same asymptotic rates of misclassification. Indeed, if  one considers the embedding given by $\mathbf{D}^{-1/2}\tilde{\mathbf{X}}$, then this embedding will be very close to the scaled Laplacian embedding and may provide a link between the two procedures. 

Note that consistent block assignments are possible using either the singular vectors or the scaled version of the singular vectors. The singular vectors themselves are essentially a whitened version of scaled singular vectors. Since the singular vectors are orthogonal, the estimated covariance of rows of the scaled vectors is proportional to the diagonal matrix given by the singular values of $\mathbf{A}$. This suggests that clustering using a criterion invariant to coordinate-scalings and rotations will likely have similar asymptotic properties. 

Critical to the proof is the bound provided by Proposition~\ref{prop:froConv}. Since this bound does not depend on the method for generating $\mathbf{Q}$, it suggests that extensions to this theorem are possible. One such extension is to take the number of blocks $K=K_n$ to go slowly to infinity. For $K_n$ growing, the parameters $\alpha$, $\beta$, and $\gamma$ are no longer constant in $n$, so we must impose conditions on these parameters. If we take $d$ fixed and assume these parameters go to 0 slowly, it is possible to allow $K_n=n^\epsilon$ for $\epsilon$ sufficiently small. Under these conditions, it can be shown that the number of incorrect block assignments is $o(n \gamma)$, which is negligible to block sizes. Our proof technique breaks down for $K_n=\Omega(n^{1/4})$ as Proposition~\ref{prop:froConv} no longer implies a gap in the singular values of $\mathbf{A}$. 

In order to avoid the model selection quagmire, we assumed in Theorem~\ref{thm:main} that the number of blocks $K$ and the latent feature dimension $d$ are known. However, the proof of this theorem suggests that both $K$ and and $d$ can be estimated consistently. Corollary~\ref{cor:singBoundA}, shows that all but $d$ of the singular values of $\mathbf{A}$ are less than $3^{1/4}n^{3/4}\log^{1/4} n$ for $n$ large enough. As discussed earlier, this shows that $\hat{d}=\max\{ i: \sigma_i(\mathbf{A}) > 3^{1/4}n^{3/4}\log^{1/4} n\}$ will be a consistent estimator for $d$.  Though this estimator is consistent, the required number of nodes for it to become accurate will depend highly on the sparsity of the graph, which controls the magnitude of the largest singular values of $A$. Furthermore, our bounds suggest that the number of nodes required for this estimate to be accurate will increase exponentially as the expected graph density decreases.

Estimating $K$ is more complicated, and we do not present a formal method to do so. We do note that the proof shows that most of the embedded vectors are concentrated around  $K$ separated points. An appropriate covering of the points by slowly vanishing balls would allow for a consistent estimate of $K$. More work is needed to provide model selection criteria which are practical to the practitioner.

Note that some practitioners may have estimates or bounds for the parameters $\mathbf{P}$ and $\rho$, derived from some prior study. In this case, provided bounds on $\alpha$, $\beta$, and $\gamma$ can be determined, the proof can be used to derive high probability bounds on the number of nodes that have been assigned to the incorrect block. This may also enable the practitioner to choose $n$ to optimize some misassignment and cost criteria.

The proofs above would remain valid if the diagonals of the adjacency matrix are modified provided that each modification is bounded. In fact, modifying the diagonals may improve the embedding to give lower numbers of misassignments. \citet{priebe11:_vertex} suggests replacing the diagonal element $\mathbf{A}_{uu}$ with $\mathrm{deg}(u)/(n-1)$ for each node $u\in[n]$. \citet{Scheinerman2010} provided an iterative algorithm to impute the diagonal. An optimal choice the diagonal is not known for general stochastic blockmodels.

Another practical concern is the possibility of missing data in the observed graph. One example may be that each edge in the true graph is only observed with probability $p$ in the observed graph. Our theory will be unaffected by this type of error since the observed graph is also distributed according to a stochastic blockmodel with edge probabilities $\mathbf{P}'=p\mathbf{P}$. As a result, asymptotic consistency remains valid. We may also allow $p$ to decrease slowly with $n$ and still achieve asymptotically negligible misassignments. However, typically the finite sample performance will if $p$ is small. 

Overall, the theory and results presented suggest that this embedding procedure is worthy of further investigation. The problems estimating $K$ and $d$, choosing between scaled and unscaled embedding and between the adjacency and the Laplacian will all be considered in future work. This work is also being generalized to more general latent position models.

Finally, under the stochastic blockmodel, our method will be less computationally demanding than ones which depend on maximizing likelihood or modularity criterion. Fast methods to compute singular value decompositions are possible, especially for sparse matrices. There are a plethora of methods for efficiently clustering points in Euclidean space. Overall, this embedding method may be valuable to the practitioner to provide a rapid method to identify blocks in networks.

\appendix
\section{Proofs of Technical Lemmas}

In this appendix, we prove the technical results stated in Section~\ref{sec:mainThm}.

\newtheorem*{lemmaSingBoundXY}{Lemma~\ref{lem:singBoundXY}}
\begin{lemmaSingBoundXY}
It almost always holds that $\alpha\gamma n\leq \sigma_d(\mathbf{XY}^T)$ and it always holds that $\sigma_{d+1}(\mathbf{XY}^T)=0$ and $\sigma_1(\mathbf{XY^T})\leq n$.
\end{lemmaSingBoundXY}
\begin{proof}
Since $\mathbf{XY}^T\in[0,1]^{n\times n}$, the nonnegative matrix $\mathbf{XY}^T(\mathbf{XY}^T)^T$ has entries bounded by $n$. The row sums are bounded by $n^2$ giving that $\sigma_1^2(\mathbf{XY}^T)=\lambda_1(\mathbf{XY}^T(\mathbf{XY}^T)^T)\leq n^2$. Since $\mathbf{X}$ and $\mathbf{Y}$ are at most rank $d$, we have $\sigma_{d+1}(\mathbf{XY})=0$.

The nonzero eigenvalues of $\mathbf{XY}^T(\mathbf{XY}^T)^T=\mathbf{XY}^T\mathbf{Y}^T\mathbf{X}$ are the same as the nonzero eigenvalues of $\mathbf{Y}^T\mathbf{YX}^T\mathbf{X}$. It almost always holds that $n_i\geq \gamma n$ for all $i$ so that 
\begin{equation}
\mathbf{X}^T\mathbf{X}=\sum_{i=1}^K n_i\nu_i\nu_i^T = \gamma n \bm{\nu}^T\bm{\nu} + \sum_{i=1}^K (n_i-\gamma n)\nu_i\nu_i^T
\label{eq:singValsXY1}
\end{equation}
is the sum of two positive semidefinite matrices, the first of which has eigenvalues all greater then $\alpha\gamma n$. This gives $\lambda_d(\mathbf{X}^T\mathbf{X})\geq \alpha\gamma n$ and similarly $\lambda_d(\mathbf{Y}^T\mathbf{Y})\geq \alpha\gamma n$. This gives that $\mathbf{Y}^T\mathbf{YX}^T\mathbf{X}$ is the product of positive definite matrices. We then use a bound on the smallest eigenvalues of the product of two positive semi-definite matrices, so that $\lambda_d(\mathbf{Y}^T\mathbf{YX}^T\mathbf{X})\geq \lambda_d(\mathbf{Y}^T\mathbf{Y})\lambda_d(\mathbf{X}^T\mathbf{X})\geq (\alpha\gamma n)^2$ \citep[Corollary 11]  {Zhang2006}. This establishes $\sigma_d^2(\mathbf{XY}^T)\geq (\alpha\gamma n)^2$. 
\end{proof}

\newtheorem*{corSingBoundA}{Corollary~\ref{cor:singBoundA}}
\begin{corSingBoundA}
It almost always holds that $\alpha\gamma n\leq \sigma_d(\mathbf{A})$ and $\sigma_{d+1}(\mathbf{A})\leq 3^{1/4}n^{3/4}\log^{1/4}n$
and it always holds that $\sigma_1(\mathbf{A})\leq n$.
\end{corSingBoundA}
\begin{proof}
First, by the same arguments as Lemma~\ref{lem:singBoundXY} we have $\sigma_1(\mathbf{A})\leq n$. By  Weyl's inequality \citep[\S 6.3]{horn85:_matrix_analy}, we have that
\begin{equation}
\begin{split}
|\sigma_i^2(\mathbf{A})-\sigma_i^2(\mathbf{XY}^T)| &= |\lambda_i(\mathbf{AA}^T)-\lambda_i(\mathbf{XY}^T(\mathbf{XY}^T)^T)|\\
&\leq \| \mathbf{AA}^T-\mathbf{XY}^T(\mathbf{XY}^T)^T\|_F.
\end{split}
\label{eq:froBound4}
\end{equation}
Together with Corollary~\ref{cor:froBound} this shows that $\sigma_{d+1}(\mathbf{A})\leq  3^{1/4}n^{3/4}\log^{1/4}n$ almost always. Since $\gamma<\rho_i$ for each $i$, Lemma~\ref{lem:singBoundXY} can be strengthened to show that there exists $\epsilon>0$, not dependent on $n$, such that $(\alpha\gamma+\epsilon) n < \sigma_d(\mathbf{XY}^T)$. Thus, we have that $(\alpha\gamma+\epsilon)^2 n^2< \sigma_d^2(\mathbf{XY}^T)$ so that $(\alpha\gamma)^2 n^2\leq\sigma_d^2(\mathbf{A})$ since $\sqrt{3}n^{3/2}\sqrt{\log n}<\epsilon^2 n^2$ for $n$ large enough. 
\end{proof}

The singular value decomposition of $\mathbf{XY}^T$ is given b $\mathbf{U}\bm{\Sigma}\mathbf{V}^T$. The next result provides bounds for the gaps between the at most $K$ distinct rows of $\mathbf{U}$ and $\mathbf{V}$. Recall that for a matrix $\mathbf{M}$, row $u$ is given $M_u^{T}$ for all $u$.

\newtheorem*{lemEigVecSep}{Lemma~\ref{lem:eigVecSep}}
\begin{lemEigVecSep} It almost always holds that, for all $u,v$ such that $X_u\neq X_v$, $\|U_u-U_v\| \geq
\beta \sqrt{\alpha \gamma}n^{-1/2}$. Similarly, for all $Y_u\neq Y_v$, $\|V_u-V_v\|\geq \beta\sqrt{\alpha \gamma}n^{-1/2}$. As a result, $\|W_u-W_v\|\geq \beta\sqrt{\alpha\gamma}n^{-1/2}$ for all $u,v$ such that $\tau(u)\neq \tau(v)$.
\end{lemEigVecSep}
\begin{proof}
 Let $\mathbf{Y}^T\mathbf{Y}=\mathbf{ED}^2\mathbf{E}^T$ for
$\mathbf{E} \in \mathbb{R}^{d \times d}$ orthogonal, $\mathbf{D} \in \mathbb{R}^{d \times d}$ diagonal.
Define $\mathbf{G}=\mathbf{XE}$, $\mathbf{G}'=\mathbf{GD}$, and $\mathbf{U}'=\mathbf{U}\bm{\Sigma}$.
Let $u,v$ be such that $X_u \neq X_v$.
From Lemma \ref{lem:singBoundXY} and its proof, diagonals of $\mathbf{D}$ are almost always at least
$\sqrt{\alpha \gamma n}$ and the diagonals of $\bm{\Sigma}$ are at most $n$.

Now, 
\begin{equation}
\begin{split}
\mathbf{G}'\mathbf{G}'^T&=
\mathbf{G} \mathbf{D}^2 \mathbf{G}^T=\mathbf{XED}^2\mathbf{E}^T\mathbf{X}^T=\mathbf{XY}^T\mathbf{YX}^T\\
&=\mathbf{U}\bm{\Sigma} \mathbf{V}^T\mathbf{V}\bm{\Sigma} \mathbf{U}^T=\mathbf{U} \bm{\Sigma}^2 \mathbf{U}^T=\mathbf{U}'\mathbf{U}'^T.
\end{split}
\label{eq:ggToUU}
\end{equation}
Let $e \in \mathbb{R}^n$ denote the vector with all
zeros except $1$ in the $u^\text{th}$ coordinate and
$-1$ in the $v^\text{th}$ coordinate. By the above we have $\|G_u'-G_v'\|^2 = e^T\mathbf{G}'\mathbf{G}'^Te=e^T\mathbf{U}'\mathbf{U}'^Te=\| U_u'-U_v'\|^2$.
Therefore we obtain that $\beta \leq \|X_u-X_v \|=\|G_u-G_v\|
\leq \frac{1}{\sqrt{\alpha \gamma n}} \|G_u'-G_v'\|=
\frac{1}{\sqrt{\alpha \gamma n}} \|U_u'-U_v'\| \leq \frac{1}{\sqrt{\alpha \gamma n}} n
\| U_u-U_v \|$, as desired.

A symmetric argument holds for $\|V_u-V_v\|$. For $\|W_u-W_v\|$ note that if $\tau(u)\neq\tau(v)$ then either $U_u\neq U_v$ or $V_u\neq V_v$.
\end{proof}

\newtheorem*{lemOne}{Lemma~\ref{lem:1}}
\begin{lemOne}
  It almost always holds that there exists an orthogonal matrix $\mathbf{R}\in\mathbb{R}^{2d\times2d}$ such that $\|\mathbf{WR}-\tilde{\mathbf{W}}\|\leq \sqrt{2}\frac{\sqrt{6}}{\alpha^2\gamma^2}\sqrt{\frac{\log n}{n}}$.
\end{lemOne}
\begin{proof}
Let $\mathcal{S}=(\frac{1}{2}\alpha^2\gamma^2n^2,\infty)$. By Lemma~\ref{lem:singBoundXY} and Corollary~\ref{cor:singBoundA}, it almost always holds that exactly $d$ eigenvalues of $\mathbf{AA}^T$ and $\mathbf{XY}^T(\mathbf{XY}^T)^T$ are in $\mathcal{S}$. Additionally, Lemma~\ref{lem:singBoundXY} shows that the gap $\delta>\alpha^2\gamma^2 n^2$. Together with Corollary~\ref{cor:froBound}, we have that 
\begin{equation}
\sqrt{2}\frac{\|\mathbf{AA}^T-\mathbf{XY}^T(\mathbf{XY}^T)^T\|_F}{\delta} \leq \sqrt{2}\frac{\sqrt{3}n^{3/2}\sqrt{\log n}}{\alpha^2\gamma^2 n^2}.
\label{eq:lemBoundDC}
\end{equation}
This shows there exists an $\mathbf{R}_1\in\mathbb{R}^{d\times d}$ such that $\|\mathbf{UR}_1-\tilde{\mathbf{U}}\|_F\leq \frac{\sqrt{6}}{\alpha^2\gamma^2}\sqrt{\frac{\log n}{n}}$.

Now note that all of the above could be repeated for $\mathbf{A}^T\mathbf{A}$ and $(\mathbf{XY}^T)^T\mathbf{XY}^T$, to find $\mathbf{R}_2\in\mathbb{R}^{d\times d}$ such that $\|\mathbf{VR}_2-\tilde{\mathbf{V}}\|_F\leq \frac{\sqrt{6}}{\alpha^2\gamma^2}\sqrt{\frac{\log n}{n}}$. Taking $\mathbf{R}$ as the direct sum of $\mathbf{R}_1$ and $\mathbf{R}_2$ gives the result. 
\end{proof} 

\section{Davis-Kahan Theorem}\label{sec:DK}
We now state and provide a brief discussion of the Davis-Kahan theorem \citep{davis70,rohe10:_spect_stoch_block_model}. First, we consider some general results from the theory of Grassmann spaces \citep{Qiu2005}.
Let ${\mathcal G}_{d,n}$ denote the set of 
$d$-dimensional subspaces of $\Re^n$. Two important metrics on  
${\mathcal G}_{d,n}$ are the {\it gap metric} $d_g$ and the 
{\it Hausdorff metric} $d_h$ which are defined as follows. 
For all ${\mathcal W}, {\mathcal W}' \in {\mathcal G}_{d,n}$, 
\begin{align}
d_g({\mathcal W},{\mathcal W}') &=\sqrt{\sum_{i=1}^d \sin^2 
\theta_i ({\mathcal W},{\mathcal W}')} \\ d_h({\mathcal W},{\mathcal W}')&=\sqrt{\sum_{i=1}^d \left(2 \sin
\frac{\theta_i ({\mathcal W},{\mathcal W}')}{2} \right)^2}
\end{align}
where 
$\theta_1({\mathcal W},{\mathcal W}')$,
$\theta_2({\mathcal W},{\mathcal W}')$, \ldots
$\theta_d({\mathcal W},{\mathcal W}')$ denote the
{\it principal angles} between ${\mathcal W}$ and
${\mathcal W}'$. By simple trigonometry $d_h({\mathcal W},{\mathcal W}') 
 \leq \sqrt{2} \cdot d_g({\mathcal W},{\mathcal W}')$.
Suppose $\mathbf{W},\mathbf{W}' \in \Re^{n,d}$ have columns which are orthonormal 
bases for ${\mathcal W}$ and  ${\mathcal W}'$, respectively.
It is well known that 
$d_h({\mathcal W},{\mathcal W}')=\min_\mathbf{R} \| \mathbf{W}\mathbf{R}-\mathbf{W}' \|_F$ where the minimum 
is over all orthogonal matrices $\mathbf{R} \in \Re^{d \times d}$.

The next theorem states the original form of the theorem from \cite{davis70} followed by the version  proved in \cite{rohe10:_spect_stoch_block_model}.
\newtheorem*{DK}{Theorem~\ref{thm:DK}}
\begin{DK}[Davis and Kahan]
Let $\mathbf{H},\mathbf{H}' \in \Re^{n \times n}$ be symmetric,
suppose ${\mathcal S} \subset \Re$ is an interval, and suppose for
some positive integer $d$ that ${\mathcal W} \in {\mathcal G}_{d,n}$
is the sum of the eigenspaces  of
$\mathbf{H}$ associated with the eigenvalues of $\mathbf{H}$ in ${\mathcal S}$, and that
${\mathcal W}' \in {\mathcal G}_{d,n}$
is the sum of the eigenspaces  of
$\mathbf{H}'$ associated with the eigenvalues of $\mathbf{H}'$ in ${\mathcal S}$.
If  $\delta$ is the minimum distance between
any eigenvalue of $\mathbf{H}$ in ${\mathcal S}$ and any
eigenvalue of $\mathbf{H}$ not in ${\mathcal S}$
then $\delta \cdot d_g({\mathcal W},{\mathcal W}') \leq \| \mathbf{H}-\mathbf{H}'\|_F$.

Furthermore, suppose $\mathbf{W},\mathbf{W}'\in \Re^{n\times d}$ are such that the columns of  $\mathbf{W}$ form an orthonormal basis for $\mathcal{W}$ and that the columns $\mathbf{W}'$ form an orthonormal basis for $\mathcal{W}'$. Then there exists an orthogonal matrix $\mathbf{R} \in \Re^{d \times d}$ such that
$\| \mathbf{W} \mathbf{R}- \mathbf{W}' \|_F \leq \frac{\sqrt{2}}{\delta}\|\mathbf{H}-\mathbf{H}'\|_F$.
\end{DK}
From the preceding analysis we see that
the version from \cite{rohe10:_spect_stoch_block_model} follows from the original theorem; indeed, we have
for some orthogonal $\mathbf{R} \in \Re^{d \times d}$ that
$\|\mathbf{W}\mathbf{R}-\mathbf{R}'\|_F = d_h({\mathcal W},{\mathcal W}') \leq \sqrt{2}
d_g({\mathcal W},{\mathcal W}') \leq  \frac{\sqrt{2}}{\delta}
 \| \mathbf{H}-\mathbf{H}'\|_F$. 

\bibliography{STFP_revision}

\begin{thebibliography}{25}
\providecommand{\natexlab}[1]{#1}
\providecommand{\url}[1]{\texttt{#1}}
\expandafter\ifx\csname urlstyle\endcsname\relax
  \providecommand{\doi}[1]{doi: #1}\else
  \providecommand{\doi}{doi: \begingroup \urlstyle{rm}\Url}\fi

\bibitem[Airoldi et~al.(2008)Airoldi, Blei, Fienberg, and Xing]{Airoldi2008}
E.~M. Airoldi, D.~M. Blei, S.~E. Fienberg, and E.~P. Xing.
\newblock {Mixed membership stochastic blockmodels}.
\newblock \emph{The Journal of Machine Learning Research}, 9:\penalty0
  1981--2014, 2008.

\bibitem[Bickel and Chen(2009)]{Bickel2009}
P.~J. Bickel and A.~Chen.
\newblock {A nonparametric view of network models and Newman-Girvan and other
  modularities.}
\newblock \emph{Proceedings of the National Academy of Sciences of the United
  States of America}, 106:\penalty0 21068--21073, 2009.

\bibitem[Choi et~al.(In press)Choi, Wolfe, and Airoldi]{Choi2010}
D.~S. Choi, P.~J. Wolfe, and E.~M. Airoldi.
\newblock {Stochastic blockmodels with growing number of classes}.
\newblock \emph{Biometrika}, In press.

\bibitem[Condon and Karp(2001)]{Condon1999}
A.~Condon and R.~M. Karp.
\newblock {Algorithms for graph partitioning on the planted partition model}.
\newblock \emph{Random Structures and Algorithms}, 18:\penalty0 116--140, 2001.

\bibitem[Davis and Kahan(1970)]{davis70}
C.~Davis and W.~Kahan.
\newblock The rotation of eigenvectors by a pertubation. {III}.
\newblock \emph{Siam Journal on Numerical Analysis}, 7:\penalty0 1--46, 1970.

\bibitem[Eckart and Young(1936)]{Eckart1936}
C.~Eckart and G.~Young.
\newblock The approximation of one matrix by another of lower rank.
\newblock \emph{Psychometika}, 1:\penalty0 211--218, 1936.

\bibitem[Fjallstrom(1998)]{Fjallstrom1998}
P.~Fjallstrom.
\newblock {Algorithms for graph partitioning: A survey}.
\newblock \emph{Computer and Information Science}, 3\penalty0 (10), 1998.

\bibitem[Fortunato(2010)]{Fortunato2010}
S.~Fortunato.
\newblock {Community detection in graphs}.
\newblock \emph{Physics Reports}, 486:\penalty0 75--174, 2010.

\bibitem[Goldenberg et~al.(2010)Goldenberg, Zheng, Fienberg, and
  Airoldi]{Goldenberg2009}
A.~Goldenberg, A.~X. Zheng, S.~E. Fienberg, and E.~M. Airoldi.
\newblock A survey of statistical network models.
\newblock \emph{Foundations and Trends in Machine Learning}, 2, 2010.

\bibitem[Handcock et~al.(2007)Handcock, Raftery, and Tantrum]{Handcock2007}
M.~S. Handcock, A.~E. Raftery, and J.~M. Tantrum.
\newblock {Model-based clustering for social networks}.
\newblock \emph{Journal of the Royal Statistical Society: Series A (Statistics
  in Society)}, 170:\penalty0 301--354, 2007.

\bibitem[Hoff et~al.(2002)Hoff, Raftery, and Handcock]{hoff02:_laten}
P.~Hoff, A.~E. Raftery, and M.~S. Handcock.
\newblock Latent space approaches to social network analysis.
\newblock \emph{Journal of the American Statistical Association}, 97:\penalty0
  1090--1098, 2002.

\bibitem[Holland et~al.(1983)Holland, Laskey, and Leinhardt]{Holland1983}
P.~W. Holland, K.~Laskey, and S.~Leinhardt.
\newblock {Stochastic blockmodels: First steps}.
\newblock \emph{Social Networks}, 5:\penalty0 109--137, 1983.

\bibitem[Horn and Johnson(1985)]{horn85:_matrix_analy}
R.~Horn and C.~Johnson.
\newblock \emph{Matrix Analysis}.
\newblock Cambridge University Press, 1985.

\bibitem[Hubert and Arabie(1985)]{Hubert1985}
L.~Hubert and P.~Arabie.
\newblock {Comparing partitions}.
\newblock \emph{Journal of Classification}, 2:\penalty0 193--218, 1985.

\bibitem[Marchette et~al.(2011)Marchette, Priebe, and
  Coppersmith]{priebe11:_vertex}
D.~J. Marchette, C.~E. Priebe, and G.~Coppersmith.
\newblock Vertex nomination via attributed random dot product graphs.
\newblock In \emph{Proceedings of the 57th ISI World Statistics Congress},
  2011.

\bibitem[McSherry(2001)]{McSherry2001}
F.~McSherry.
\newblock {Spectral partitioning of random graphs}.
\newblock In \emph{Proceedings of the 42nd IEEE symposium on Foundations of
  Computer Science}, pages 529--537. IEEE Computer Society, 2001.

\bibitem[Newman and Girvan(2004)]{Newman2004}
M.~Newman and M.~Girvan.
\newblock {Finding and evaluating community structure in networks}.
\newblock \emph{Physical Review}, 69:\penalty0 1--15, 2004.

\bibitem[Nowicki and Snijders(2001)]{Nowicki2001}
K.~Nowicki and T.~A.~B. Snijders.
\newblock {Estimation and Prediction for Stochastic Blockstructures}.
\newblock \emph{Journal of the American Statistical Association}, 96:\penalty0
  1077--1087, 2001.

\bibitem[Qi et~al.(2005)Qi, Zhang, and Li]{Qiu2005}
L.~Qi, Y.~Zhang, and C.-K Li.
\newblock {Unitarily Invariant Metrics on the Grassmann Space}.
\newblock \emph{SIAM Journal on Matrix Analysis and Applications}, 27:\penalty0
  507--531, 2005.

\bibitem[Rohe et~al.(2011)Rohe, Chatterjee, and
  Yu]{rohe10:_spect_stoch_block_model}
K.~Rohe, S.~Chatterjee, and B.~Yu.
\newblock Spectral clustering and the high-dimensional stochastic blockmodel.
\newblock \emph{Annals of Statistics}, 39:\penalty0 1878--1915, 2011.

\bibitem[Scheinerman and Tucker(2010)]{Scheinerman2010}
E.~Scheinerman and K.~Tucker.
\newblock {Modeling graphs using dot product representations}.
\newblock \emph{Computational Statistics}, 25:\penalty0 1--16, 2010.

\bibitem[Snijders and Nowicki(1997)]{snijders97:_estim}
T.~Snijders and K.~Nowicki.
\newblock Estimation and prediction for stocchastic block models for graphs
  with latent block structure.
\newblock \emph{Journal of Classification}, 14:\penalty0 75--100, 1997.

\bibitem[Wang and Wong(1987)]{Wang1987}
Y.~J. Wang and G.~Y. Wong.
\newblock {Stochastic Blockmodels for Directed Graphs}.
\newblock \emph{Journal of the American Statistical Association}, 82:\penalty0
  8--19, 1987.

\bibitem[Young and Scheinerman(2007)]{young07:_random}
S.~Young and E.~Scheinerman.
\newblock Random dot product models for social networks.
\newblock In \emph{Proceedings of the 5th international conference on
  algorithms and models for the web-graph}, pages 138--149, 2007.

\bibitem[Zhang and Zhang(2006)]{Zhang2006}
F.~Zhang and Q.~Zhang.
\newblock {Eigenvalue inequalities for matrix product}.
\newblock \emph{Automatic Control, IEEE Transactions on}, 51:\penalty0
  1506--1509, 2006.

\end{thebibliography}

\end{document}